\documentclass[letterpaper, 10 pt, conference]{styles/ieeeconf}
\IEEEoverridecommandlockouts
\usepackage[utf8]{inputenc}
\usepackage[T1]{fontenc}
\usepackage[english]{babel}

\usepackage{xspace}

\usepackage{url}
\usepackage{color}
\usepackage{wrapfig}
\usepackage{placeins}
\usepackage{subcaption}
\usepackage{placeins}
\usepackage{cuted}
\usepackage{capt-of}
\newif\iffigs
\figstrue %

\usepackage[fleqn]{amsmath}
\usepackage{float}
\usepackage{setspace}
\usepackage{graphicx}
\usepackage{mathrsfs}
\usepackage{amssymb}
\usepackage{nicefrac}
\usepackage[linesnumbered,ruled,noend]{algorithm2e}

\SetCommentSty{mycommfont}
\SetKwComment{Comment}{// }{}
\SetKwInput{KwPar}{Parallelism}
\SetArgSty{textnormal}
\SetAlCapSkip{1em}

\let\oldnl\nl
\newcommand{\nonl}{\renewcommand{\nl}{\let\nl\oldnl}}
\makeatletter
\newcommand{\removelatexerror}{\let\@latex@error\@gobble}
\makeatother

\usepackage{tabularx}
\usepackage{multirow}
\usepackage{booktabs}
\usepackage{colortbl}
\usepackage{siunitx}

\usepackage{etoolbox}
\makeatletter
\patchcmd{\@makecaption}
{\scshape}
{}
{}
{}
\makeatother

\let\labelindent\relax
\usepackage[inline]{enumitem}

\usepackage{listings}
\lstnewenvironment{itemlisting}[1][]
{%
  \mbox{}
  \vspace*{-\baselineskip}
  \lstset{
    xleftmargin=\leftmargin,
    linewidth=\linewidth,
    #1
  }%
}
{}
\usepackage{multicol}
\usepackage{footmisc}
\usepackage{xcolor}

\usepackage{amsmath, amssymb, amsthm}
\theoremstyle{plain}
\newtheorem{theorem}{Theorem}
\newtheorem{lemma}{Lemma}
\newtheorem{proposition}{Proposition}
\theoremstyle{definition}

\theoremstyle{remark}

\definecolor{startgreen}{HTML}{4CB863}
\definecolor{goalred}{HTML}{EFA8A9}

\usepackage{lipsum}
\usepackage{xcolor}
\usepackage{subcaption}
\usepackage{csquotes}
\usepackage[
  maxbibnames=99,
  maxcitenames=2,
  natbib=true,
  style=numeric-comp,
  backend=biber,
  sorting=none,
  giveninits=true,
  url=false,
  doi=false,
  eprint=false,
  isbn=false,
]{biblatex}

\usepackage{xurl}

\usepackage[pdfa,colorlinks,bookmarksopen,bookmarksnumbered,allcolors=black,urlcolor=blue]{hyperref}

\usepackage[
  activate   = {true},
  protrusion = false,
  expansion  = true,
  kerning    = true,
  spacing    = true,
  tracking   = false,
  auto       = true,
  selected   = true,
  factor     = 2000,
  stretch    = 20,
  shrink     = 20,
]{microtype}

\usepackage[nameinlink,capitalise]{cleveref}
\crefname{line}{line}{lines}
\crefname{figure}{Fig.}{Figs.}
\Crefname{figure}{Fig.}{Figs.}
\crefname{equation}{Eq.}{Eqs.}
\Crefname{equation}{Eq.}{Eqs.}
\crefname{section}{Sec.}{Secs.}
\Crefname{section}{Sec.}{Secs.}
\crefname{definition}{Def.}{Defs.}
\Crefname{definition}{Def.}{Defs.}
\crefname{theorem}{Thm.}{Thms.}
\Crefname{theorem}{Thm.}{Thms.}
\crefname{lemma}{Lem.}{Lems.}
\Crefname{lemma}{Lem.}{Lems.}
\crefname{proposition}{Prop.}{Props.}
\Crefname{proposition}{Prop.}{Props.}
\crefname{algorithm}{Alg.}{Algs.}
\Crefname{algorithm}{Alg.}{Algs.}
\crefname{algocf}{Alg.}{Algs.}
\Crefname{algocf}{Alg.}{Algs.}
\crefname{AlgoLine}{Ln.}{Lns.}
\Crefname{AlgoLine}{Ln.}{Lns.}
\crefname{assumption}{Asm.}{Asms.}
\Crefname{assumption}{Asm.}{Asms.}
\crefname{subassumption}{Asm.}{Asms.}
\Crefname{subassumption}{Asm.}{Asms.}
\Crefname{problem}{Problem}{Problems}
\crefname{problem}{Problem}{Problems}

\newcommand{\simd}{\textsc{simd}\xspace}
\newcommand{\cpu}{\textsc{cpu}\xspace}
\newcommand{\gpu}{\textsc{gpu}\xspace}

\newcommand{\gpus}{\textsc{gpu}s\xspace}

\newcommand{\sbmps}{\textsc{sbmp}s\xspace}
\newcommand{\rrt}{\textsc{rrt}\xspace}
\newcommand{\rrtconnect}{\textsc{rrt}-Connect\xspace}
\newcommand{\rrtstar}{\textsc{rrt}{\scriptspace=0pt\(^*\)}\xspace}

\newcommand{\dof}{\textsc{d\scalebox{.8}{o}f}\xspace}

\newcommand{\simt}{\textsc{simt}\xspace}
\newcommand{\urdf}{\textsc{urdf}\xspace}
\newcommand{\io}{\textsc{i}/\textsc{o}\xspace}
\newcommand{\nvidia}{\textsc{nvidia}\xspace}
\newcommand{\prm}{\textsc{prm}\xspace}
\newcommand{\cbs}{\textsc{cbs}\xspace}
\newcommand{\asao}{\textit{a.s.a.o.}\@\xspace}
\newcommand{\aox}{\mbox{\textsc{ao}-x}\xspace}
\newcommand{\aorrtc}{\textsc{aorrtc}\xspace}
\newcommand{\bitstar}{\textsc{bit}{\scriptspace=0pt\(^*\)}\xspace}
\newcommand{\fmt}{\textsc{fmt}\xspace}
\newcommand{\drrt}{d\textsc{rrt}\xspace}
\newcommand{\drrtstar}{d\textsc{rrt}{\scriptspace=0pt\(^*\)}\xspace}
\newcommand{\aodrrt}{ao-d\textsc{rrt}\xspace}
\newcommand{\prrt}{\textsc{prrt}\xspace}
\newcommand{\prrtc}{p\textsc{rrtc}\xspace}
\newcommand{\quadrrt}{Quad-\textsc{rrt}\xspace}
\newcommand{\capt}{\textsc{capt}\xspace}
\newcommand{\vcc}{\textsc{vcc}\xspace}
\newcommand{\kinopax}{Kino-\textsc{pax}\xspace}
\newcommand{\vampmr}{\textsc{vamp}-\textsc{mr}\xspace}
\newcommand{\curobovtwo}{cuRoboV2\xspace}
\newcommand{\mrpop}{\mbox{\textsc{mr.~pop}}\xspace}
\newcommand{\foam}{\mbox{foam}\xspace}

\newcommand{\mbm}{\textsc{mbm}\xspace}

\newcommand{\curobo}{cuRobo\xspace}

\newcommand{\vamp}{\textsc{vamp}\xspace}

\definecolor{pRRTC_color_threads}{HTML}{328535}%

\definecolor{pRRTC_color_groups}{HTML}{ff77b4}

\newcommand{\setadd}{\mathrel{\vphantom{\leftarrow}\smash{\overset{\scriptscriptstyle+}{\leftarrow}}}}

\DeclareMathOperator*{\argmin}{argmin}

\usepackage{flushend}
\usepackage{comment}

\newif\ifanon

\ifanon
\else
\fi

\title{\LARGE \bf MR. POP: Multi-Robot Parallel Optimizing Planner for\\Almost-Surely Asymptotically Optimal Planning}

\ifanon
    \author{Anonymous Authors\(^{1}\)
    }
\else
\author{Chih H. Huang\(^{1,2}\), Roy Xing\(^{2}\), Brian Plancher\(^{2}\), and Zachary Kingston\(^{3}\)%
  \thanks{This material is based upon work supported by the National Science Foundation (under Awards 2411369 and 2236868). Any opinions, findings, conclusions, or recommendations expressed in this material are those of the authors and do not necessarily reflect those of the funding organizations.}%
  \thanks{\(^{1}\)Columbia College, Columbia University {\tt\footnotesize}}%
  \thanks{\(^{2}\)Department of Computer Science, Dartmouth College {\tt\footnotesize \{chih.gr,roy.xing.gr,brian.k.plancher\}@dartmouth.edu}}%
  \thanks{\(^{3}\)Department of Computer Science, Purdue University {\tt\footnotesize zkingston@purdue.edu}}%
}
\fi

\begin{document}
\raggedbottom
\maketitle
\thispagestyle{empty}
\pagestyle{empty}

\ifanon
\setcounter{footnote}{3}
\fi

\vspace*{-150mm}
\begin{strip}
\centering
\captionsetup{type=figure}
\subfloat[Problem \protect\textcolor{startgreen}{start} and \protect\textcolor{goalred}{goal}]{\includegraphics[width=0.24\linewidth]{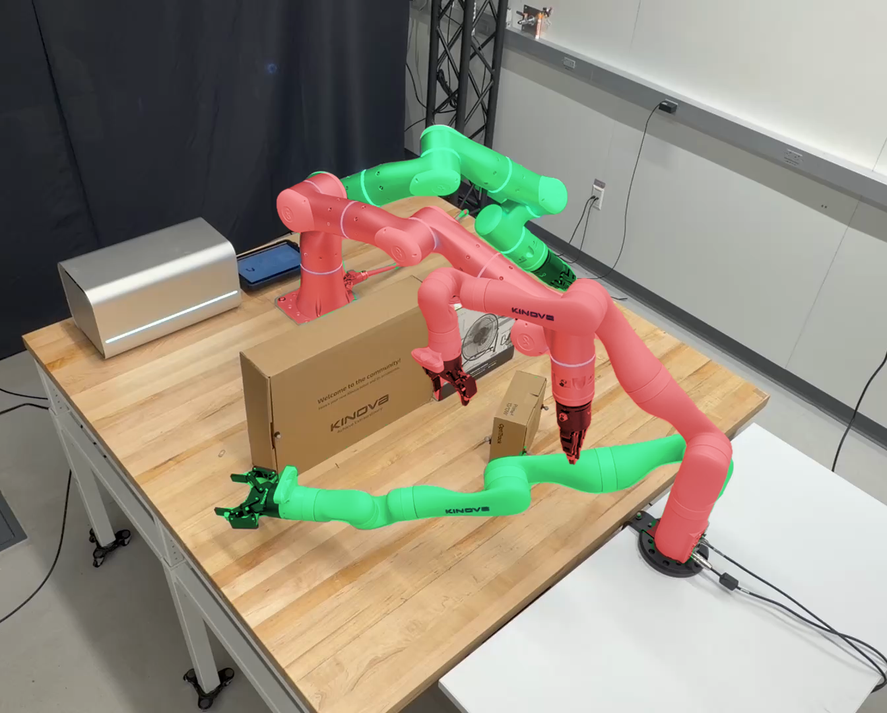}%
\label{fig:a}}
\hfil
\subfloat[\drrt]{\includegraphics[width=0.24\linewidth]{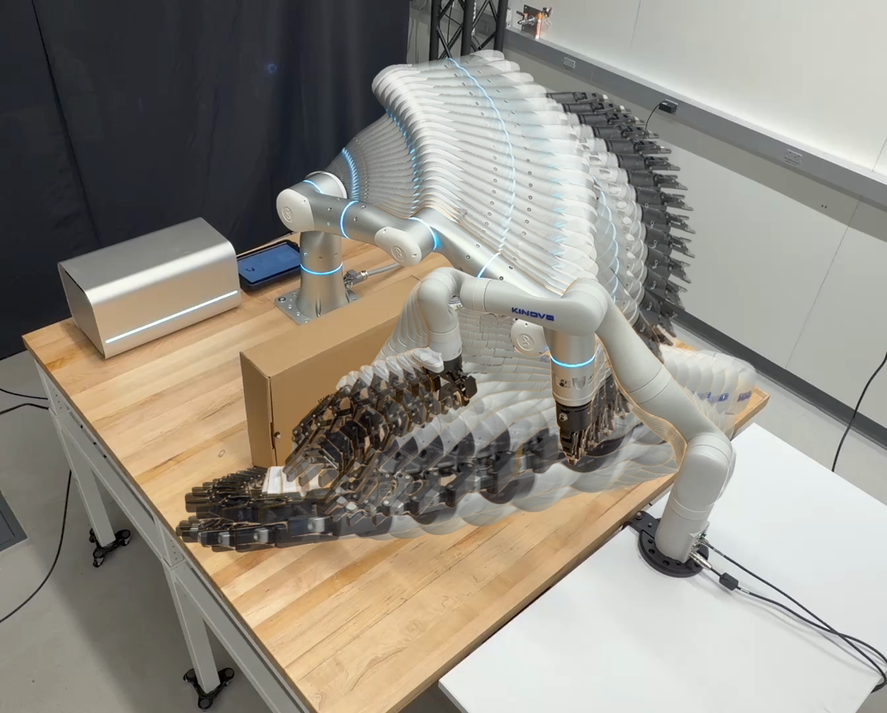}%
\label{fig:b}}
\hfil
\subfloat[\mrpop (ours)]{\includegraphics[width=0.24\linewidth]{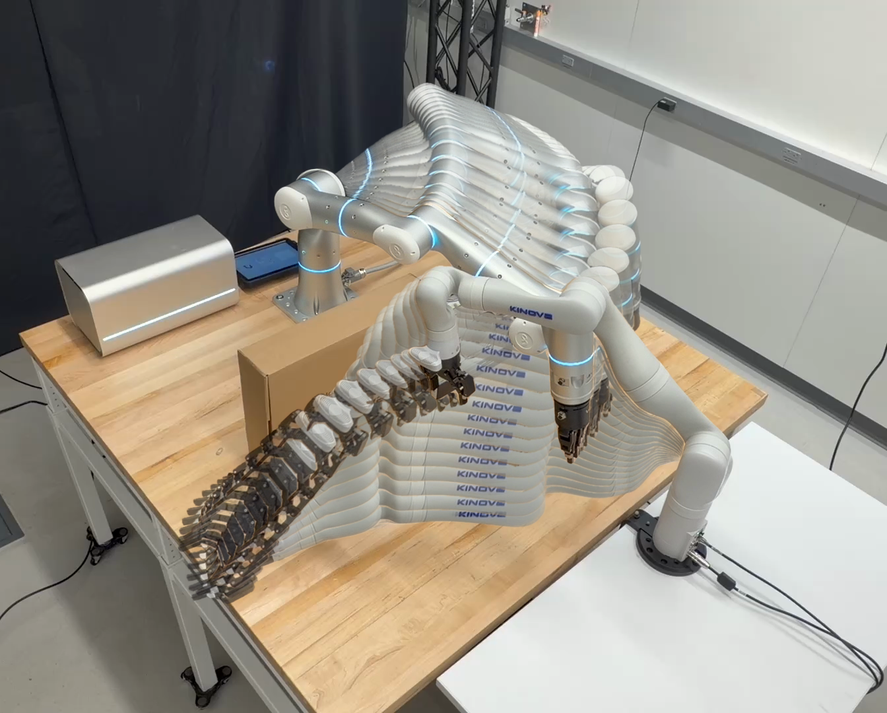}%
\label{fig:c}}
\hfil
\subfloat[\mrpop (ours) + \curobovtwo]{\includegraphics[width=0.24\linewidth]{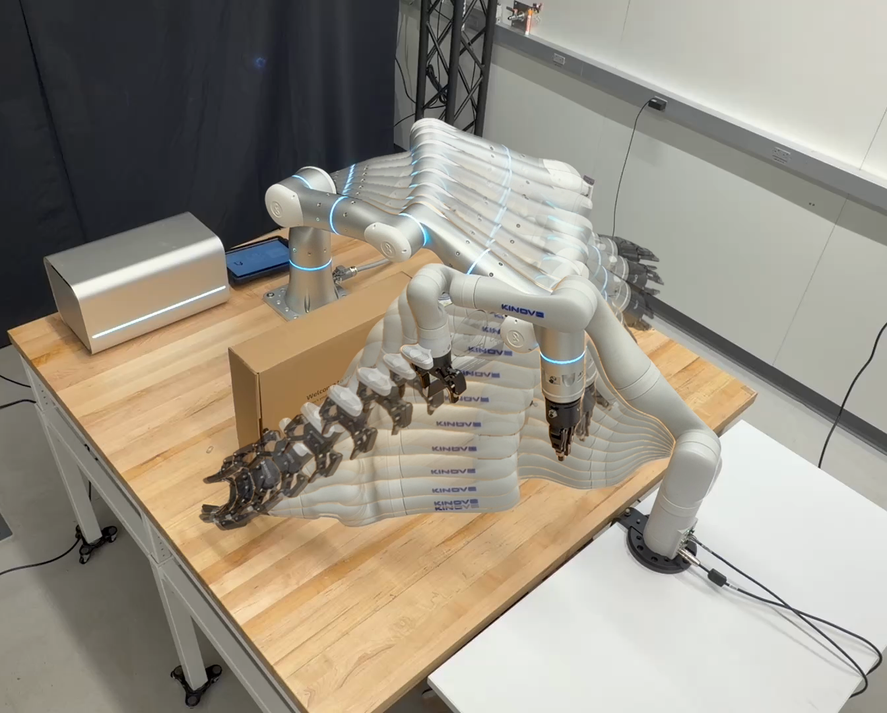}%
\label{fig:d}}
\caption{Example paths planned by \drrt, \mrpop, and \mrpop + \curobovtwo on a 14-\dof multi-robot system consisting of the 7-\dof Flexiv Rizon 4s and the 7-\dof Kinova Gen3. \mrpop leads to paths of lower cost and acts as an effective seeder for motion optimizers, raising their success rates (e.g., from 4\% to 72\%) by creating high-quality, diverse seeds that help avoid local minima.}
\label{fig:hardware_comparison}
\vspace{-10pt}
\end{strip}

\begin{abstract}

  Finding globally optimal paths remains a fundamental challenge in multi-robot motion planning. Despite acceleration of almost-surely asymptotically optimal (\asao) planners via \cpu-based parallelism, achieving both probabilistic convergence guarantees and strong computational performance, these algorithms still struggle to scale to multi-robot settings. 
As such, we introduce \mrpop, a \gpu-based \asao multi-robot planner based on \drrt and the \aox meta-algorithm. 
\mrpop uses large-scale \gpu-based \simt-parallelism to simultaneously run hundreds of roadmap construction and tree search iterations with underlying parallel nearest neighbor search and collision checking operations.
We show that this enables \mrpop to become the only planner achieving a 100\% solve rate while being faster than state-of-the-art \asao planners in multi-robot systems up to 35-\dof. \mrpop also raises the success rate of downstream motion optimizers (e.g., from 4\% to 72\%), by creating high-quality, diverse seeds that help avoid local minima.

\end{abstract}

\section{Introduction} \label{sec:intro}
Motion planning aims to find low-cost, collision-free paths through a robot's configuration space~\cite{lavalle2006planning, kavraki2016motion,orthey2023sampling}. This problem can be separated into two stages: 1) constructing an initial path and 2) optimizing it for a better solution. Sampling-based motion planners (\sbmps), such as \rrt and its bidirectional variant, \rrtconnect, are well known for their ability to efficiently find initial paths~\cite{lavalle2001rapidly, kuffner2000rrt}. 
Almost-surely asymptotically optimal (\asao) planners, e.g., \rrtstar~\cite{karaman2011sampling} and \bitstar~\cite{gammell2015batch}, attempt to also optimize paths by continuously improving solution quality until a user-set termination condition is met, which converge to the global optimum almost surely or in probability. Recent efforts have also been made to accelerate \asao planners algorithmically by utilizing the asymptotically optimal x (\aox) meta-algorithm~\cite{hauser2016asymptotically} and computationally through \cpu-based \simd parallelism (primarily for collision checking~\cite{wilson2025aorrtc}). 

At the same time, current multi-robot sampling-based planners, such as \drrt~\cite{solovey2016finding} and its anytime \asao counterpart \drrtstar~\cite{shome2020drrt}, avoid planning in the multi-robot composite space due to computational challenges, and instead decouple the problem into first finding collision-free paths for individual robots before resolving inter-robot collisions. Despite this decoupling, computational challenges still prevent \drrt from scaling to large teams of manipulators, and the additional overhead of the iterative \asao process remains a key barrier for real-time use in multi-robot settings.

As such, we introduce \textbf{\mrpop}, a \textbf{m}ulti-\textbf{r}obot \textbf{p}arallel \textbf{o}ptimizing \textbf{p}lanner based on the \aox meta-algorithm and \drrt that leverages \gpu-parallelism for efficient multi-robot \asao planning. We take advantage of the fact that modern \gpus are capable of running tens of thousands of simultaneous threads via \simt parallelism and use this scalability for running hundreds of  roadmap construction and tree search iterations simultaneously, while each iteration also runs the underlying parallelized nearest neighbor search and collision checking operations.
We evaluate \mrpop against state-of-the-art feasible and \asao planners; \mrpop is the only planner that consistently achieves a 100\% solve rate in finding initial solutions for multi-manipulator settings and shows faster convergence to optimal solutions. %
We also demonstrate the effectiveness of \mrpop as an efficient, high-quality seed for multi-robot jerk-free motion through optimizers such as \curobovtwo~\cite{sundaralingam2026curobov2} (e.g., raising the optimizer's success rate from 4\% to 72\% on multi-robot tasks). 
We also ablate the components of \mrpop in single-robot settings to demonstrate the performance of its search and \aox layers.
Finally, we deploy \mrpop on a 14 \dof dual-arm setup as depicted in Fig.~\ref{fig:hardware_comparison} for real-world validation. We release \mrpop open source: 
\ifanon
    \texttt{redacted for double blind review} 
\else
    \texttt{\url{https://github.com/CoMMALab/MR.POP}}
\fi

\section{Background and Related Work} \label{sec:related}

The \aox meta-algorithm provides existing feasible planners, which return any constraint-satisfying solution without consideration for path quality, with \asao guarantees~\cite{hauser2016asymptotically}.
Given an \(n\)-dimensional problem, it adds an additional dimension describing the cost-to-come. %
It then attempts to generate trajectories with an increasingly lower cost bound. 
In particular, \aorrtc combines the \aox meta-algorithm with the \rrtconnect planner and has shown orders of magnitude speedups in finding paths of lower sub-optimality factor~\cite{wilson2025aorrtc}. 

\subsection{Parallel Acceleration of Planning Algorithms}
Parallelism enables real-time planning performance and can be divided into three levels: high, running multiple instances of the planning algorithm simultaneously; medium, running multiple iterations of the algorithm in parallel; and low, parallelizing primitive operations (e.g., collision checks)~\cite{huang2025prrtc}.

There is a rich literature on \cpu parallel planning. For example, at the high-level, C-Forest runs several instances of the planning algorithm simultaneously and achieves superlinear speedup~\cite{otte2013c}. For mid-level parallelism, \prrt runs each \rrt iteration in independent threads. For low-level parallelism, \vamp~\cite{thomason2024motions} redesigns sampling, forward kinematics tracing, and collision checking to utilize \simd parallelism and achieves microsecond-scale planning. \capt~\cite{ramsey2024} and \vcc~\cite{chen2026vcc} further accelerate pointcloud-based collision checking by building parallelism-friendly data structures.

Recent work has also demonstrated the benefits of large-scale \gpu \simt parallelism in planning.
\quadrrt provides a 10x speedup through high-level parallelism, growing four trees in parallel~\cite{hidalgo2018quad}.
For mid-level parallelism, \kinopax expands multiple nodes simultaneously~\cite{perrault2025kino}.
For low-level parallelism, \curobo redesigns kinematics computation and collision checking to be carried out by multiple threads in parallel~\cite{sundaralingam2023curobo}.
Finally, \prrtc uses a combination of mid- and low-level parallelism to achieve 5-10x speedups over both prior \cpu- and \gpu-parallel approaches~\cite{huang2025prrtc}.

\subsection{Multi-Robot Motion Planning}
Robotic applications often require multiple manipulators to collaborate in the same workspace, necessitating simultaneous planning. One approach is to plan directly in the composite space of all robots, but this does not scale to large teams of manipulators. Decoupled and factored methods instead plan per-robot and resolve conflicts between plans~\cite{solis2021representation, guo2026efficient}. \drrt~\cite{solovey2016finding} decouples the problem by first constructing a collision-free roadmap for each individual robot, then searching over the tensor product of all roadmaps to avoid inter-robot collision. \drrt retains feasibility under the multi-robot setting, leading to the anytime \asao counterpart \drrtstar~\cite{shome2020drrt}, which repeats \drrt iterations to reach the globally optimal multi-robot solution in the limit. Attempts to speed up \drrtstar~\cite{solano2023fast} pre-compute inter-robot collisions and avoid costly rewiring.

Despite these advances, for high-\dof multi-robot systems \drrt and \drrtstar can still take up to hundreds of seconds to find initial solutions and converge to low-cost paths~\cite{shome2020drrt}. As such, we propose combining the \aox meta-algorithm with \drrt, while redesigning the algorithms to utilize \simt parallelism on the \gpu at all three levels, for real-time globally optimal multi-robot planning.

\iffigs
\begin{figure*}[t]
   \centering
   \vspace{1em}
   \includegraphics[width=1\textwidth]{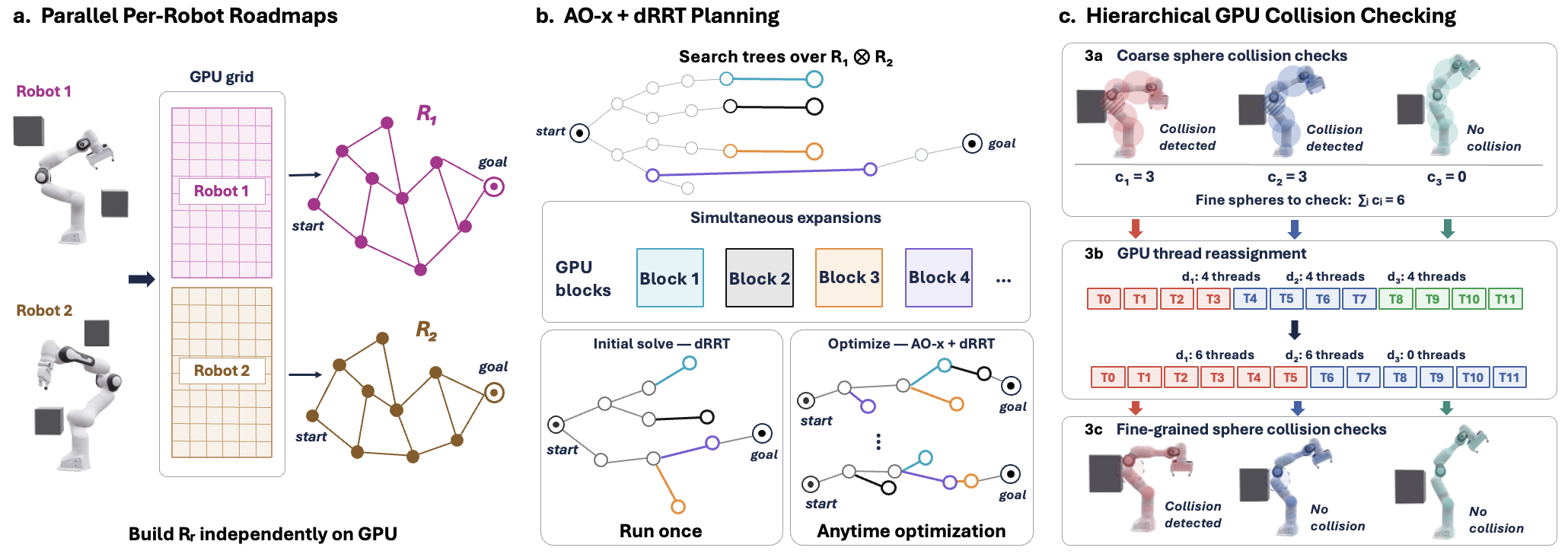}
   \vspace{-5pt}
   \caption{Overview of \mrpop. (a) The \gpu grid is divided evenly between the number of robots for roadmap construction. (b) \mrpop builds search trees over the tensor product of each individual robot's roadmap, and every \gpu block carries out its own \drrt or \aox + \drrt iteration. (c) For the primitive collision checking operations, 4 threads are assigned to each discretized motion for collision checking against the robot's coarse spherization. For fine-grained collision checking, the threads are reassigned between the discretized motions based on the coarse results and the estimated workload.}%
   \label{fig:pp_overview}
   \vspace{-15pt}
\end{figure*}
\fi

\section{MR. POP Algorithm}
\label{sec:design}
The overall structure of \mrpop is shown in~\cref{alg:mrpop} and Fig.~\ref{fig:pp_overview}. During the initial solve, the algorithm functions as a \gpu-parallelized version of bi-directional \drrt. For optimization, \mrpop searches over the roadmap tensor product by pairing bi-directional \drrt with the \aox meta-algorithm. The operations that benefit from \simt parallelism at the block level are highlighted in \textcolor{blue}{blue}, while operations parallelized at the thread level are highlighted in \textcolor{purple}{purple}. Block level parallelism is used to run hundreds of planning iterations simultaneously, as each block handles its own planning iteration. This is seen in both roadmap construction and bi-directional tree search. Thread level parallelism is used for primitive operations such as sampling, nearest neighbor search, and collision checking.

\subsection{Problem Setup}
\label{sec:setup}
Robot \(i\) has configuration space \(\mathbb{C}_i \subseteq \mathbb{R}^{d}\) and free space \(\mathbb{C}^{\mathrm{f}}_i\).
A composite configuration is \(q = (q^1, \dots, q^K)\), and the composite free space \(\mathbb{C}^{\mathrm{f}} \subseteq \mathbb{C}_1 \times \cdots \times \mathbb{C}_K\) excludes robot--robot collisions.
\mrpop plans from \(q_{\mathrm{start}}\) to \(q_{\mathrm{goal}}\) in \(\mathbb{C}^{\mathrm{f}}\).
Robot \(i\)'s roadmap \(\mathcal{R}_i = (V_i, E_i)\) is a \textsc{prm} built from \(n\) samples with connection radius \(\delta\).
The composite roadmap \(\mathbb{G}\) is the tensor product of the \(\mathcal{R}_i\), where vertices are mutually collision-free and its edges correspond to each robot either traversing one of its own edges or staying put and the motion is collision-free.
A path on \(\mathbb{G}\) is a vertex sequence \(q_0, \dots, q_M\) with cost \(C(\sigma) = \sum_j \|q_{j+1} - q_j\|\), its composite path length.
\(\underline{g}(q) = \|q - q_{\mathrm{start}}\|\) is a lower bound on the cost-to-come of \(q\), and \(c_{\mathrm{best}}\) is the cost of the best path found so far.
We write \(S \setadd e\) for adding edge \(e\) to a set \(S\).

\subsection{Roadmap Construction}
Roadmap construction, the \textsc{roadmap} subroutine of~\cref{alg:subroutines}, is used in both initial solve and path optimization. The process happens in parallel across the \gpu grid, and each robot's roadmap is generated simultaneously. Upon entering the function (Ln. 2), each block is assigned a robot id \(r\) and contributes to robot \(r\)'s roadmap. Several blocks may build one roadmap, and one of them is the leader. During setup, the leader initializes the roadmap with the given start and goal configurations of robot \(r\), while resetting the \(\mathit{connected}_r\) flag for that roadmap (Ln. 3-6).

For \(n\) iterations, each of the \(B\) blocks assigned to robot \(r\) samples a collision-free configuration (Ln. 11), searches for roadmap neighbors within radius \(\delta\) (Ln. 13), and collision-checks the edges to those neighbors for robot \(r\) (Ln. 15). If the edges are collision-free, they are added to the roadmap for robot \(r\) (Ln. 16). After \(n\) iterations, the \gpu performs a grid-wise sync (Ln. 17) so every block's writes to roadmap \(r\) reach the leading block. The leading block then checks connectivity by breadth-first search (Ln. 18-19).

\subsection{Solve for Initial Solution}
During the initial solve, \mrpop functions as a \gpu-parallel bi-directional \drrt. Each block of the \gpu grid carries out an independent bi-directional \drrt iteration, and since modern \gpus are capable of running hundreds of blocks simultaneously, \mrpop is capable of carrying out hundreds of \drrt iterations in the time of a single iteration.

At the beginning of each iteration, each block selects the smaller tree for expansion. Each iteration then follows the \drrt structure, drawing random configurations \(q_{\mathrm{rand}}\) from the composite space, searching for nearest neighbor \(q_{\mathrm{near}}\) in the tree, and selecting the new sample \(q_{\mathrm{new}}\) by querying \drrt's oracle~\cite{solovey2016finding}. Afterwards, if the edge connecting sample \(q_{\mathrm{new}}\) and  nearest neighbor \(q_{\mathrm{near}}\) is collision-free, then node \(q_{\mathrm{new}}\) and edge \(\{q_{\mathrm{new}}, q_{\mathrm{near}}\}\) are added to the tree. \mrpop then attempts to connect the collision-free sample \(q_{\mathrm{new}}\) to its nearest neighbor in the opposing tree \(q_{\mathrm{conn}}\). If the connection is collision-free, then an initial solution is found.

\begin{figure}[!t]
  \removelatexerror
  \begin{algorithm}[H]
    \caption{\mrpop}\label{alg:mrpop}
    \linespread{1.15}\small
    \DontPrintSemicolon
    \SetInd{0.5em}{0.5em}%
    \SetKwProg{func}{Function}{}{}%

    \func{\upshape\texttt{MR.\,POP}(\(q_{\mathrm{start}}, q_{\mathrm{goal}}, \delta, n\))}{
      \KwPar{High: bidirectional \drrt/\aox\\
        \textcolor{blue}{Mid: block-parallel iterations}\\
      \textcolor{purple}{Low: thread-parallel primitives}}
      \(c_{\mathrm{best}} \gets \infty\),\quad \(\sigma_{\mathrm{best}} \gets \emptyset\),\quad \(\phi \gets\) \texttt{true}\;
      \While{\(t < t_{\max}\)}{
        \(\mathcal{R} \gets\) \textcolor{blue}{\texttt{ROADMAP}}(\(q_{\mathrm{start}}, q_{\mathrm{goal}}, \phi, \delta, n\))\;
        \uIf(\Comment*[f]{initial solve}){\(\phi\)}{
          \(\sigma \gets\) \textcolor{blue}{\texttt{dRRT\_SEARCH}}(\(\mathcal{R}, q_{\mathrm{start}}, q_{\mathrm{goal}}\))\;
          \(c_{\mathrm{best}} \gets C(\sigma)\),\quad \(\sigma_{\mathrm{best}} \gets \sigma\)\;
          \(\phi \gets\) \texttt{false}\;
        }
        \Else(\Comment*[f]{optimizing}){
          \(\sigma \gets\) \textcolor{blue}{\texttt{AO-X}}(\(\mathcal{R}, c_{\mathrm{best}}\))\;
          \(c_{\mathrm{best}} \gets C(\sigma)\),\quad \(\sigma_{\mathrm{best}} \gets \sigma\)\;
        }
      }
      \Return{\(\sigma_{\mathrm{best}}\)}\;
    }{}
  \end{algorithm}%
\end{figure}

\begin{figure}[!t]
  \removelatexerror
  \begin{algorithm}[H]
    \caption{\mrpop \gpu subroutines.}\label{alg:subroutines}
    \linespread{1.15}\small
    \DontPrintSemicolon
    \SetInd{0.5em}{0.5em}%
    \SetKwProg{func}{Function}{}{}%

    \func{\upshape\textcolor{blue}{\texttt{ROADMAP}}(\(q_{\mathrm{start}}, q_{\mathrm{goal}}, \phi, \delta, n\))}{
      \(r \gets \lfloor \texttt{bid} / B \rfloor\) \Comment*[r]{robot id}
      \If(\Comment*[f]{one per robot}){\texttt{IS\_LEADER}(\texttt{bid}, \(r\))}{
        \(\mathit{connected}_r \gets\) \texttt{false}\;
        \If{\(\phi\)}{
          \(V_r \setadd q_{\mathrm{start}}\),\quad \(V_r \setadd q_{\mathrm{goal}}\)\;
        }
      }
      \texttt{GRID\_SYNC}()\;
      \While{\(\exists\, r' : \neg\, \mathit{connected}_{r'}\)}{
        \If{\(\neg\, \mathit{connected}_r\)}{
          \For{\(i \gets 1\) \KwTo \(n\)}{
            \(c \gets\) \textcolor{purple}{\texttt{SAMPLE\_FREE}}(\(r\))\;
            \(V_r \setadd c\)\;
            \(\mathcal{N} \gets\) \textcolor{purple}{\texttt{NEIGHBORS}}(\(V_r, c, \delta\))\;
            \ForEach{\(c' \in \mathcal{N}\)}{
              \If{\(\textcolor{purple}{(c, c') \in \mathbb{C}^{\mathrm{f}}_r}\)}{
                \(E_r \setadd (c, c')\)\;
              }
            }
          }
        }
        \texttt{GRID\_SYNC}() \Comment*[r]{merge block writes}
        \If{\texttt{IS\_LEADER}(\texttt{bid}, \(r\))}{
          \(\mathit{connected}_r \gets\) \texttt{BFS}(\(\mathcal{R}_r, q_{\mathrm{start}}, q_{\mathrm{goal}}\))\;
        }
        \texttt{GRID\_SYNC}()\;
      }
      \Return{\(\{\mathcal{R}_r\}\)}\;
    }{}

    \vspace{0.5em}
    \func{\upshape\textcolor{blue}{\texttt{AO-X}}(\(\mathcal{R}, c_{\mathrm{best}}\))}{
      \While(\Comment*[f]{in parallel}){\(\neg\, \mathit{solved}\)}{
        \If(\Comment*[f]{smaller tree}){\(|V_a| > |V_b|\)}{
          \texttt{SWAP}(\(T_a, T_b\))\;
        }
        \(q_{\mathrm{rand}} \gets\) \textcolor{purple}{\texttt{INFORMED\_SAMPLE}}(\(c_{\mathrm{best}}\))\;
        \(c_{\mathrm{rand}} \sim \mathcal{U}[\underline{g}(q_{\mathrm{rand}}),\ c_{\mathrm{best}}]\)\;
        \(z_{\mathrm{near}} \gets\) \textcolor{purple}{\texttt{NEAREST}}(\(V_a, q_{\mathrm{rand}}, c_{\mathrm{rand}}\))\;
        \(q_{\mathrm{new}} \gets\) \textcolor{purple}{\texttt{ORACLE\_AOX}}(\(\mathcal{R}, z_{\mathrm{near}}, q_{\mathrm{rand}}, c_{\mathrm{rand}}\))\;
        \If{\(\textcolor{purple}{(q_{\mathrm{near}}, q_{\mathrm{new}}) \in \mathbb{C}^{\mathrm{f}}}\)}{
          \(z_{\mathrm{new}} \gets (q_{\mathrm{new}},\ g_{\mathrm{near}} + \lVert q_{\mathrm{new}} - q_{\mathrm{near}} \rVert)\)\;
          \(V_a \setadd z_{\mathrm{new}}\),\quad \(E_a \setadd (z_{\mathrm{near}}, z_{\mathrm{new}})\)\;
          \(z_{\mathrm{conn}} \gets\) \textcolor{purple}{\texttt{NEAREST}}(\(V_b, q_{\mathrm{new}}, c_{\mathrm{best}} - g_{\mathrm{new}}\))\;
          \If{\(z_{\mathrm{conn}} \neq \emptyset \wedge \textcolor{purple}{(q_{\mathrm{new}}, q_{\mathrm{conn}}) \in \mathbb{C}^{\mathrm{f}}}\)}{
            \(\mathit{solved} \gets\) \texttt{true}\;
          }
        }
      }
      \Return{\texttt{EXTRACT\_PATH}(\(T_a, T_b\))}\;
    }{}

    \vspace{0.5em}
    \func{\upshape\textcolor{purple}{\texttt{NEAREST}}(\(V, q, c\))}{
      \(\mathcal{A} \gets \{ (q_v, g_v) \in V \mid g_v + \lVert q - q_v \rVert < c \}\)\;
      \Return{\(\argmin_{(q_v, g_v) \in \mathcal{A}} \lVert q - q_v \rVert\)}\;
    }{}
  \end{algorithm}%
\end{figure}

\subsection{Converge to Global Optimality}
\label{sec:convergence}
After finding an initial solution, \mrpop optimizes by pairing the \aox meta-algorithm with \drrt, the \textsc{ao-x} subroutine of~\cref{alg:subroutines}. \aox plans in \emph{state-cost space}~\cite{hauser2016asymptotically}: a tree vertex is a pair \(z = (q, g)\) of a configuration and its cost-to-come, a sample is a pair \(z_{\mathrm{rand}} = (q_{\mathrm{rand}}, c_{\mathrm{rand}})\) of a configuration and a cost budget, so finding a path cheaper than \(c_{\mathrm{best}}\) is a feasibility search in this space. As in the initial solve, every block runs its own iteration in parallel.

Each iteration draws the two coordinates of \(z_{\mathrm{rand}}\). The configuration \(q_{\mathrm{rand}}\) comes from informed sampling~\cite{gammell2014informed} (Ln. 26): given \(c_{\mathrm{best}}\), it is drawn from the prolate hyperspheroid containing every configuration that could lie on a path cheaper than \(c_{\mathrm{best}}\), which \mrpop projects into in parallel across threads. The cost coordinate \(c_{\mathrm{rand}}\) is drawn uniformly from \([\underline{g}(q_{\mathrm{rand}}), c_{\mathrm{best}}]\), the budgets under which \(q_{\mathrm{rand}}\) is reachable (Ln. 27). The nearest vertex \(z_{\mathrm{near}} = (q_{\mathrm{near}}, g_{\mathrm{near}})\) is the tree vertex closest to \(q_{\mathrm{rand}}\) among those that can reach it within the budget, \(g_{\mathrm{near}} + \|q_{\mathrm{rand}} - q_{\mathrm{near}}\| < c_{\mathrm{rand}}\), with ties between vertices at the same configuration broken toward the lowest cost (Ln. 28, 37--39). Unlike \drrt's oracle, \mrpop's considers only the \emph{admissible} neighbors \(v\) of \(q_{\mathrm{near}}\), those whose extension \((v,\ g_{\mathrm{near}} + \|v - q_{\mathrm{near}}\|)\) stays within \(c_{\mathrm{rand}}\), and among them selects the one best aligned with \(q_{\mathrm{rand}}\) (Ln. 29). This test discards only extensions that cannot have a cost-to-come cheaper than \(c_{\mathrm{rand}}\), and the random cost coordinate keeps every path cheaper than \(c_{\mathrm{best}}\) reachable with positive probability (\cref{sec:analysis}). To connect to the other tree, \aox applies the same nearest rule to the opposing tree with budget \(c_{\mathrm{best}} - g_{\mathrm{new}}\) (Ln. 33), so any collision-free connection yields an improved path.

\subsection{Multi-Thread Reassignment for Collision Checking}
Collision checking is the main computational bottleneck in sampling-based planning~\cite{bialkowski2011massively}. Within roadmap construction, \drrt, and \aox iterations, collision checking exploits \simt parallelism. Each robot's \urdf is first fed through \foam~\cite{coumar2025foam} to generate coarse and fine-grained spherical approximations of its collision geometry. To perform a collision check on an edge, \mrpop discretizes the edge into multiple motions and carries out collision checking on each checkpoint.

There are two stages to collision checking, as shown in Fig.~\ref{fig:pp_overview}. During the first stage, four threads are assigned to each discretized motion to carry out forward kinematics tracing and collision checking on the coarse representation in parallel, similar to the collision checking kernel described in~\cite{huang2025prrtc}. For the coarse representation, each robot link is approximated by a conservative sphere. Suppose there are \(t\) discretized motions \(d_1, d_2 ... d_t\), then \(4t\) threads would be carrying out the first stage collision checking process simultaneously, with threads \(t_1, t_2, t_3, t_4\) working on motion \(d_1\) and threads \(t_5, t_6, t_7, t_8\) working on motion \(d_2\), etc. During this stage, each discretized motion \(d_i\) maintains a count \(c_i\) on how many fine-grained collision spheres need to be checked in the second stage. Suppose a robot link \(x\) of discretized motion \(d_i\) is found to be in collision using the coarse representation by thread \(t_{4i}\), and in the fine-grained representation link \(x\) is approximated by \(w_x\) spheres, then thread \(t_{4i}\) would add \(w_x\) to the count \(c_i\). At the end of the first stage, \(c_1, c_2 ... c_t\) would hold the number of fine-grained spheres that need to be checked for each discretized motion during the second stage.

For the second stage, again four threads are assigned to carry out forward kinematics tracing for the fine-grained representation. Afterwards, each thread is reassigned to carry out collision checking based on the (\(c_i : total \ spheres \ to \ check\)) ratio, and each discretized motion \(d_i\)'s assigned thread count is proportional to the number of fine-grained spheres motion \(d_i\) needs to have checked. The collision sphere count \(c_i\) is then divided so that each assigned thread of motion \(d_i\) has an equal workload and the threads carry out collision checking in parallel, preventing stalling and increasing throughput. 

\subsection{Seeding for Multi-Robot Jerk-Free Motion}
In addition to providing globally optimal paths, \mrpop can be used as a seeder for producing multi-robot jerk-free motion through locally optimal optimizers such as \curobovtwo~\cite{sundaralingam2026curobov2}. \mrpop's probabilistic global optimality, through iteratively shrinking cost bounds of the \aox meta-algorithm, means that it produces diverse seeds throughout its optimization process, helping locally optimal optimizers escape local minima.
To exemplify this capability, we merged \mrpop into \curobovtwo's seeding component. During seeding,  \mrpop records the multiple paths produced through the optimization process. Then, the paths are used as seeds for \curobovtwo's jerk-free motion optimization, with paths of lower cost taking priority in the case of having more paths than the required number of seeds.

\section{Theoretical Analysis} \label{sec:analysis}
In this section we show that \mrpop is \asao
\(\mathbb{G}_n\) is the composite roadmap built from \(n\) samples per robot, and \(C^\ast_n\) is the least cost of a \(q_{\mathrm{start}}\) to \(q_{\mathrm{goal}}\) path in \(\mathbb{G}_n\).
\(C^\ast\) is the infimum of cost over paths from \(q_{\mathrm{start}}\) to \(q_{\mathrm{goal}}\) in \(\mathbb{C}^{\mathrm{f}}\).

\begin{theorem}[]
  \label{thm:asao}
  Suppose the problem is \emph{robustly optimal}~\cite{karaman2011sampling}: for each \(\epsilon > 0\), some path of cost at most \((1+\epsilon) C^\ast\) has positive clearance.
  Then \mrpop is \asao: its incumbent converges to \(C^\ast\) almost surely as the roadmap size \(n\) and the iteration count \(k\) tend to infinity.
\end{theorem}

\begin{proof}
  Fix \(n\).
  Each optimizing iteration searches state-cost space for a path below \(c_{\mathrm{best}}\) (\cref{sec:convergence}), the bounded suboptimality problem of \aox~\cite{hauser2016asymptotically}, \cref{thm:wellbehaved}.
  By \cref{lem:cond1,lem:cond2} it satisfies \cref{thm:wellbehaved}, so the incumbent converges to \(C^\ast_n\) almost surely, and since \(\mathbb{G}_n\) admits finitely many path costs, it equals \(C^\ast_n\) after finitely many iterations.
  By \cref{prop:radius}, \(C^\ast_n \rightarrow C^\ast\) almost surely.
  Hence for each \(j \in \mathbb{N}\), almost surely \(C^\ast_n \le C^\ast + 1/j\) for \(n \ge n_j\) and \(c_{\mathrm{best}} = C^\ast_n\) for \(k > k_j(n)\), so \(\limsup c_{\mathrm{best}} \le C^\ast + 1/j\) on any schedule with \(n \rightarrow \infty\) and \(k > k_j(n)\); as \(c_{\mathrm{best}} \ge C^\ast\), \(c_{\mathrm{best}} \rightarrow C^\ast\) almost surely.
\end{proof}

\begin{theorem}[Hauser and Zhou~\cite{hauser2016asymptotically}]
  \label{thm:wellbehaved}
  The \aox meta-algorithm is \asao whenever its feasible subroutine is \emph{well-behaved}, that is, on a problem with optimal cost \(C^\ast_n\):
  \begin{enumerate}[label=\textbf{(C\arabic*)}, leftmargin=*, align=left, labelsep=0.4em]
    \item it terminates in almost surely finite time whenever a solution below \(c_{\mathrm{best}}\) exists; and
    \item its expected suboptimality contracts by a fixed fraction each iteration,
  \end{enumerate}
  \begin{equation}
    \mathbb{E}\!\left[C(y) \mid c_{\mathrm{best}}\right] - C^\ast_n \;\le\; (1-w)\,(c_{\mathrm{best}} - C^\ast_n), w > 0.
    \label{eq:contraction}
  \end{equation}
\end{theorem}

We instantiate the subroutine with cost-bounded \drrt on the fixed roadmap \(\mathbb{G}_n\) and verify both conditions.

\begin{lemma}[\textbf{C1}]
  \label{lem:cond1}
  If \(\mathbb{G}_n\) has a path cheaper than \(c_{\mathrm{best}}\), one optimizing iteration returns one in almost surely finite time.
\end{lemma}

\begin{proof}
  Let \(\sigma\) be such a path with \(M\) edges and prefix costs \(g_0 < \dots < g_M = C(\sigma)\).
  The step from \(q_j\) to \(q_{j+1}\) passes the filter and admissibility test whenever \(c_{\mathrm{rand}} > g_{j+1}\), with probability at least \(p_n = (c_{\mathrm{best}} - C(\sigma)) / c_{\mathrm{best}}\) for any \(q_{\mathrm{rand}}\), and \(q_{\mathrm{rand}}\) falls in the Voronoi cell of \(q_j\) and the oracle's cone toward \(q_{j+1}\) with probability at least \drrt's advancement constant \(\xi_n\)~\cite{solovey2016finding}.
  Both are bounded below over the finitely many path costs in \(\mathbb{G}_n\), so within \(M\) iterations the tree reaches \((q_{\mathrm{goal}}, g')\) with \(g' \le C(\sigma)\) with probability at least \((p_n \xi_n)^M\), hence almost surely in finite time.
  The connection step accepts only paths strictly cheaper than \(c_{\mathrm{best}}\).
\end{proof}

\begin{lemma}[\textbf{C2}]
  \label{lem:cond2}
  On \(\mathbb{G}_n\), the subroutine satisfies \eqref{eq:contraction} with a constant \(w_n > 0\) that depends only on \(\mathbb{G}_n\).
\end{lemma}

\begin{proof}
  \(\mathbb{G}_n\) is finite and edge costs are nonnegative, so the costs of its simple \(q_{\mathrm{start}}\) to \(q_{\mathrm{goal}}\) paths, including those closed by connecting segments, form a finite set \(\mathcal{C}_n\) whose least element is \(C^\ast_n\), and the incumbent takes values in \(\mathcal{C}_n\).
  For \(c_{\mathrm{best}} \in \mathcal{C}_n\) above \(C^\ast_n\), \cref{lem:cond1} returns a path strictly cheaper than \(c_{\mathrm{best}}\) almost surely, hence of cost at most the next element of \(\mathcal{C}_n\) below it.
  The gap to \(C^\ast_n\) therefore shrinks by at least \(\gamma(c_{\mathrm{best}}) = c_{\mathrm{best}} - \max\{c \in \mathcal{C}_n \mid c < c_{\mathrm{best}}\} > 0\).
  Let \(w_n\) be the minimum of \(\gamma(c) / (c - C^\ast_n)\) over the finitely many \(c \in \mathcal{C}_n\) above \(C^\ast_n\); it is positive.
  Then
  \begin{align*}
    \mathbb{E}[C(y) \mid c_{\mathrm{best}}] - C^\ast_n &\le c_{\mathrm{best}} - \gamma(c_{\mathrm{best}}) - C^\ast_n \\
    &\le (1 - w_n)(c_{\mathrm{best}} - C^\ast_n). \qedhere
  \end{align*}
\end{proof}

\begin{proposition}
  \label{prop:radius}
  Let each roadmap use connection radius \(r(n) \ge r^\ast(n)\), where
  \begin{equation}
    r^\ast(n) = (1+\eta)\,2\!\left(\!1+\tfrac{1}{d}\right)^{\frac{1}{d}}
    \!\left(\frac{\max_i \mu(\mathbb{C}^{\mathrm{f}}_i)}{\zeta_d}\right)^{\frac{1}{d}}
    \!\left(\frac{\log n}{n}\right)^{\frac{1}{d}}
    \label{eq:radius}
  \end{equation}
  for fixed \(\eta > 0\), with \(\mu\) the volume measure and \(\zeta_d\) the volume of the unit \(d\)-ball.
  Then \(C^\ast_n \rightarrow C^\ast\) almost surely.
\end{proposition}

\begin{proof}
  Eq.~\eqref{eq:radius} is the \textsc{prm}\(^\ast\) radius~\cite{karaman2011sampling} applied per robot.
  A composite path with positive clearance projects to a path with the same clearance for each robot.
  A composite vertex lies near a composite waypoint exactly when every robot's vertex lies near its own waypoint, so the roadmap's miss probability is at most the sum of \(K\) per-robot miss probabilities, each at the single-robot exponent.
  Per-robot paths combine because stay-put edges have zero length.
  Hence \(\mathbb{G}_n\) contains a path of cost at most \((1+\epsilon) C^\ast\) with probability tending to one, for every \(\epsilon > 0\)~\cite{shome2020drrt}.
  Since the roadmaps are constructed incrementally, \(C^\ast_n\) is monotonically nonincreasing. Moreover, \(C^\ast_n \ge C^\ast\) and for every \(\epsilon > 0\), \(C^\ast_n \le (1+\epsilon)C^\ast\) with probability tending to one. Thus, \(C^\ast_n \to C^\ast\) almost surely.

\end{proof}

\mrpop grows the roadmap between iterations rather than fixing \(n\), and it uses a constant radius \(\delta\) rather than the schedule in \eqref{eq:radius}.
Neither affects \cref{thm:asao}, as growth only adds edges, and \(\delta\) exceeds \(r^\ast(n)\) for all \(n\) past some \(n_0\).

\begin{figure*}[t]
  \centering
  \begin{subfigure}[t]{0.28\textwidth}
    \centering
    \includegraphics[width=\linewidth,keepaspectratio]{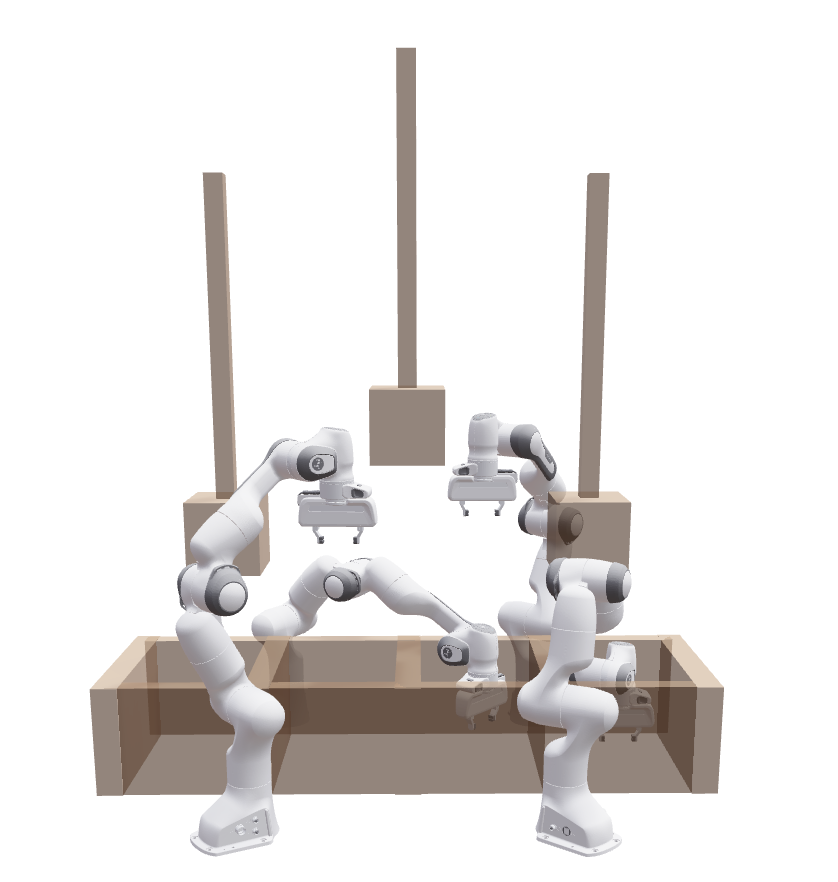}
    \caption{}
    \label{fig:franka4_scene}
  \end{subfigure}
  \hfill
  \begin{subfigure}[t]{0.34\textwidth}
    \centering
    \includegraphics[width=\linewidth,keepaspectratio]{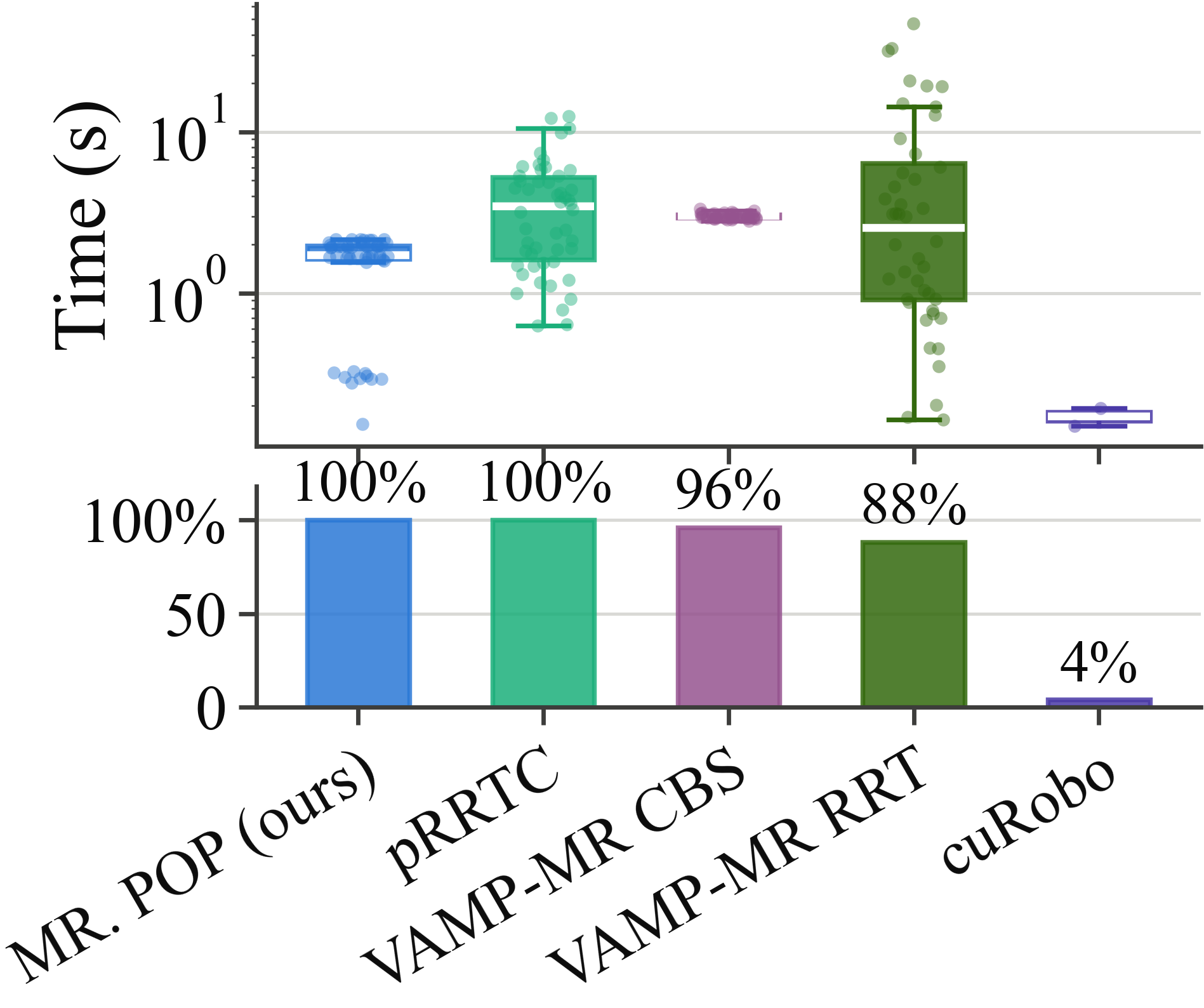}
    \caption{}
    \label{fig:panda4_init}
  \end{subfigure}
  \hfill
  \begin{subfigure}[t]{0.35\textwidth}
    \centering
    \includegraphics[width=\linewidth,keepaspectratio]{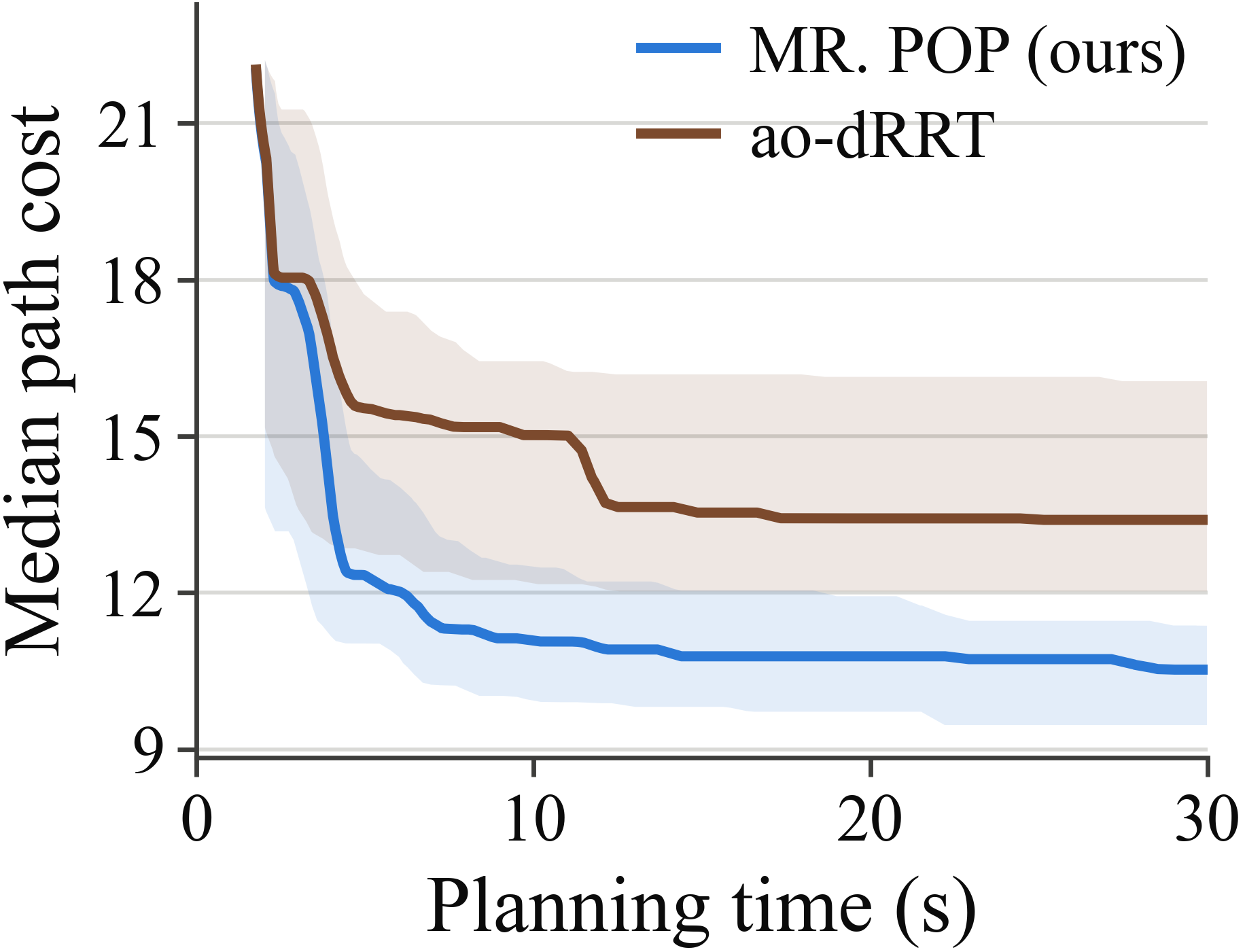}
    \caption{}
    \label{fig:cost_vs_time_panda4}
  \end{subfigure}
  \hfill
  \caption{Benchmark results for the 28 \dof four Franka bin packing problems.
    \textbf{(a)} The bin packing scene provides 50 manipulation problems.
    \textbf{(b)} Initial solution time and solve rate. Only \mrpop and \prrtc achieved a 100\% solve rate, and among all planners with a solve rate over 60\%, \mrpop averages a 2.5x speedup over \prrtc, a 1.9x speedup against \vampmr \cbs, and a 4.3x speedup against \vampmr \rrt.  
    \textbf{(c)} Optimization time vs median path cost and interquartile range. \mrpop converges to solutions of lower cost faster than \aodrrt while having a lower interquartile range, indicating a more consistent performance.}
  \label{fig:franka4}
  \vspace{-5pt}
\end{figure*}

\begin{figure*}[t]
  \centering
  \begin{subfigure}[t]{0.28\textwidth}
    \centering
    \includegraphics[width=\linewidth,keepaspectratio]{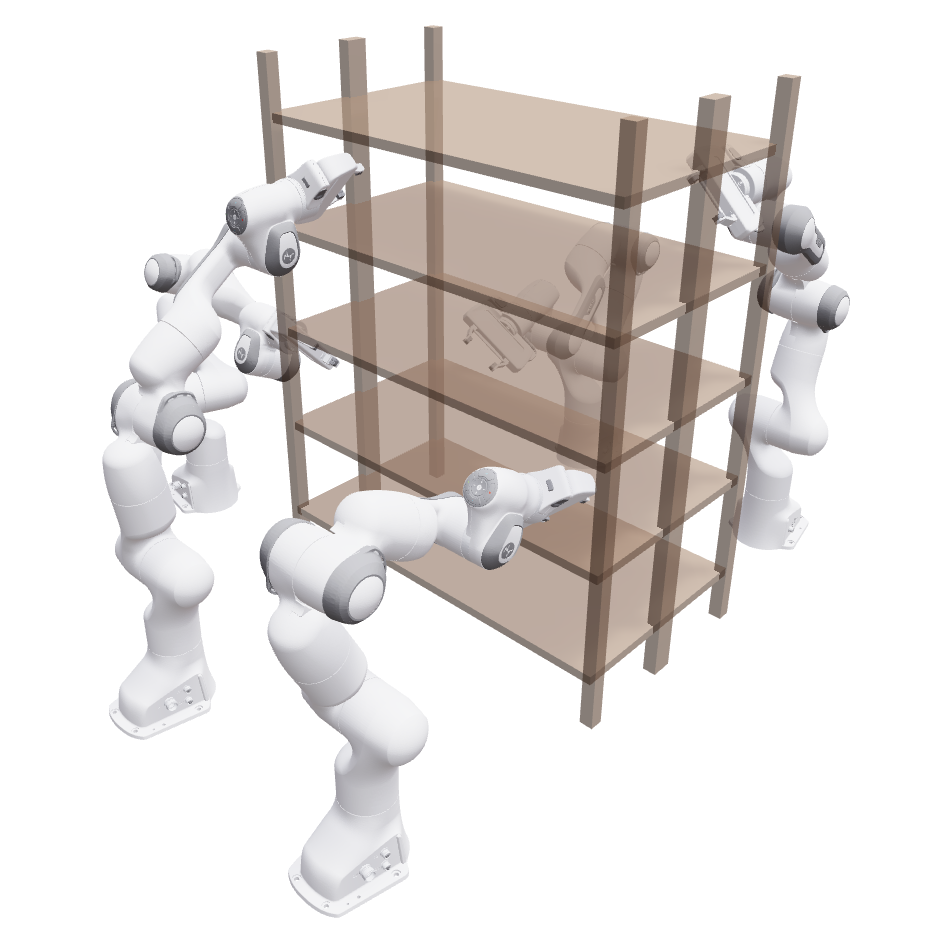}
    \caption{}
    \label{fig:franka5_scene}
  \end{subfigure}
  \hfill
  \begin{subfigure}[t]{0.34\textwidth}
    \centering
    \includegraphics[width=\linewidth,keepaspectratio]{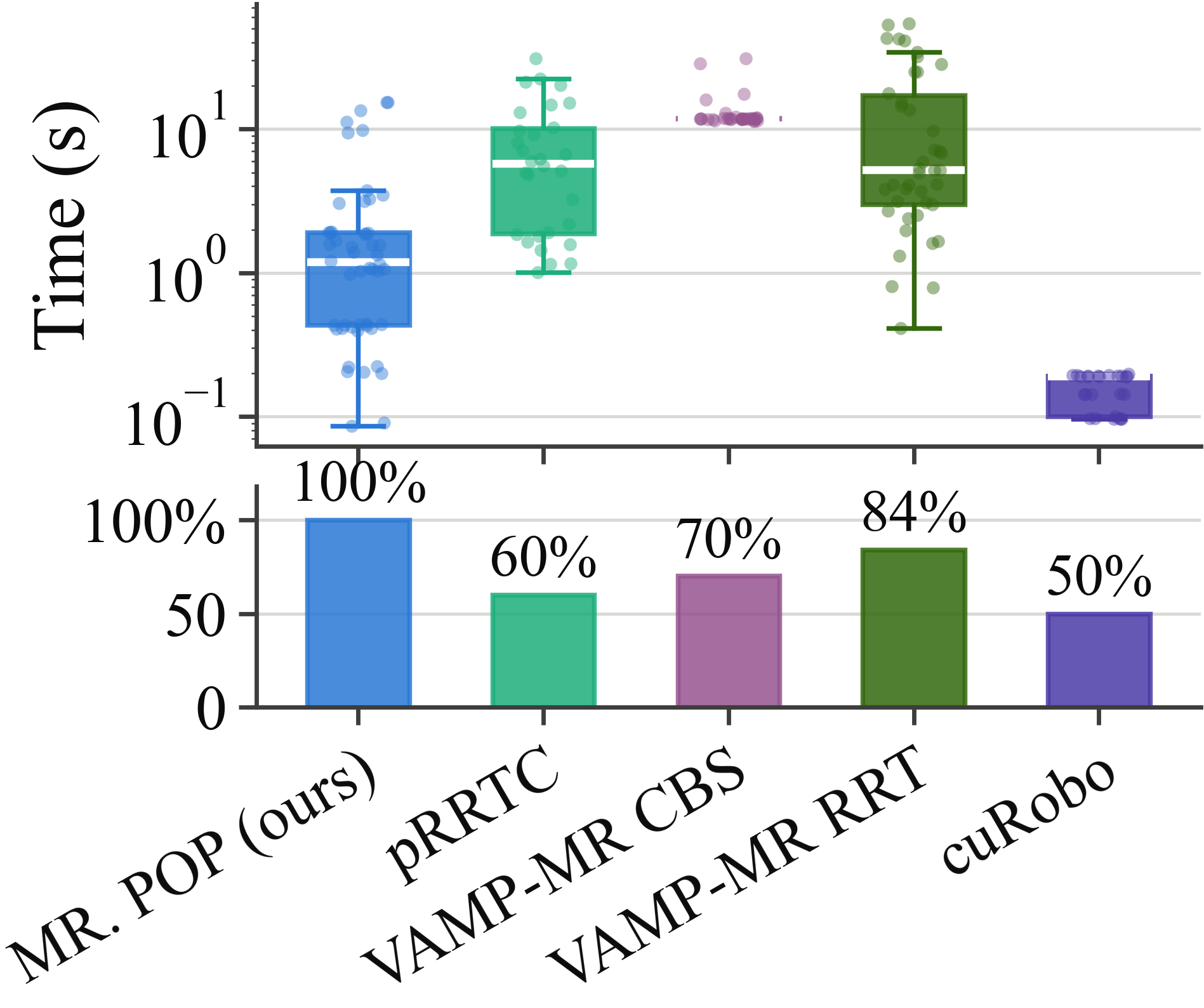}
    \caption{}
    \label{fig:panda5_init}
  \end{subfigure}
  \hfill
  \begin{subfigure}[t]{0.34\textwidth}
    \centering
    \includegraphics[width=\linewidth,keepaspectratio]{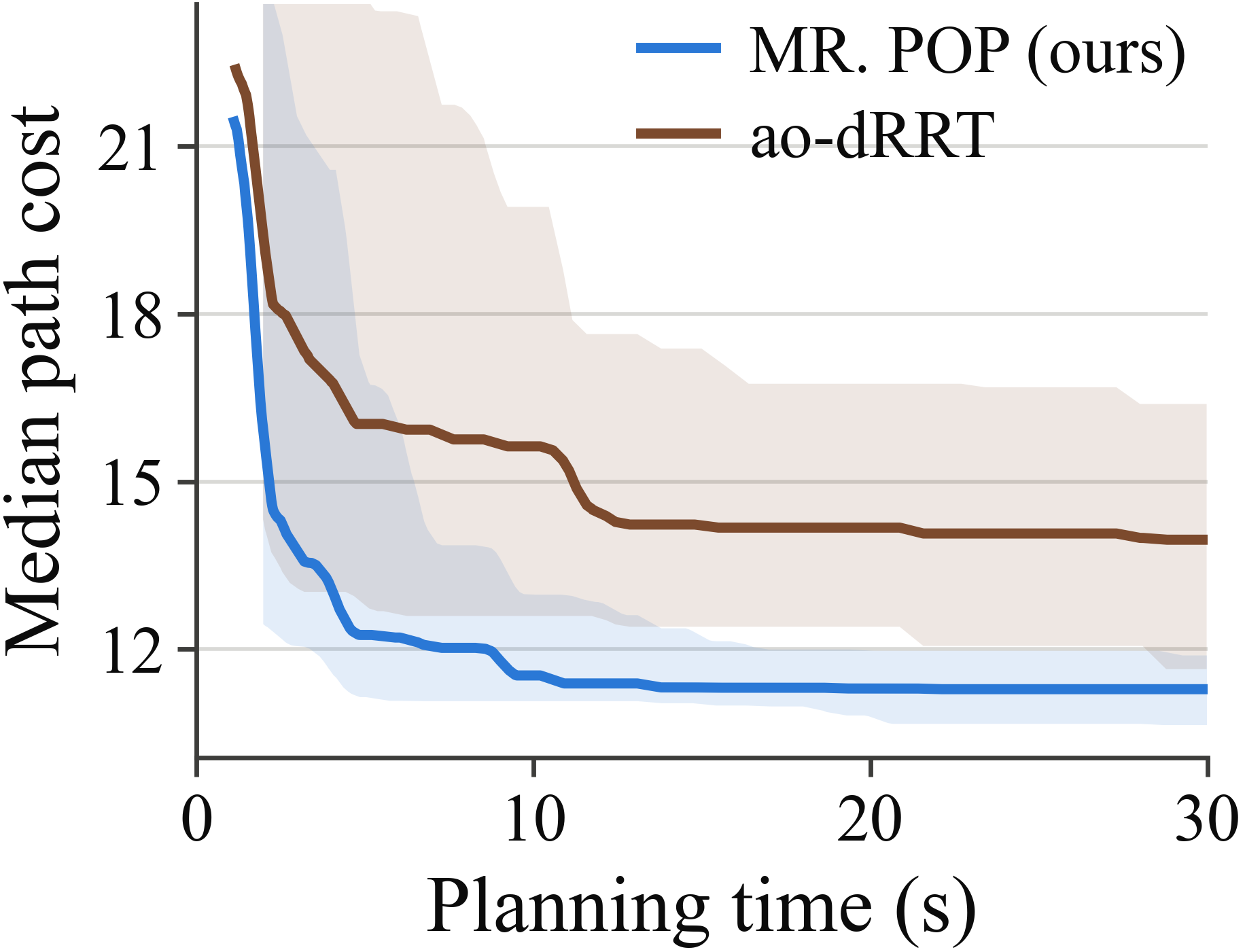}
    \caption{}
    \label{fig:cost_vs_time_panda5}
  \end{subfigure}
  \hfill
  \caption{Benchmark results for the 35 \dof five Franka shelf reaching problems.
    \textbf{(a)} The shelf reaching scene provides 50 manipulation problems.
    \textbf{(b)} Initial solution time and solve rate. Only \mrpop achieved a 100\% solve rate, and among all planners with a solve rate of at least 60\%, \mrpop averages a 3.1x speedup over \prrtc and a 5.1x speedup against \vampmr \cbs and \vampmr \rrt.  
    \textbf{(c)} Optimization time vs median path cost and interquartile range. \mrpop converges to solutions of lower cost faster than \aodrrt while having a lower interquartile range, indicating a more consistent performance.}
  \label{fig:franka5}
  \vspace{-5pt}
\end{figure*}

\section{Experiments} \label{sec:experiments}

In this section, we analyze \mrpop's performance in multi-robot and single-robot problems.
For multi-robot problems, \mrpop is evaluated on finding initial solutions and path optimization. We also combine \mrpop with \curobovtwo to show how globally optimal seeding from \mrpop improves the reliability of such optimizing planners for multi-robot problems.
\mrpop's tree search and \aox components are also isolated to ablate their performance using single-robot tasks. There we both evaluate the effectiveness of \mrpop's bi-directional tree search in finding initial solutions, as well as \mrpop's \aox layer's effectiveness at path optimization.
Finally, we validate \mrpop on dual-arm hardware.

\subsection{Methodology}
Multi-robot experiments adapt a prior multi-robot benchmark~\cite{shaoul2024accelerating} which includes a set of fifty problems where four Franka arms do bin-packing, totaling 28 \dof in the scene (Fig.~\ref{fig:franka4}a), and a set of fifty problems that have five Franka arms perform shelf reaching, totaling 35 \dof in the scene (Fig.~\ref{fig:franka5}a). 
Single-robot ablation experiments use the 8 \dof Fetch from the MotionBenchMaker (\mbm) dataset~\cite{chamzas2021motionbenchmaker}, with 7 scenes of 100 problems including counter-top manipulation, shelf operation, and constrained reaching (\cref{fig:fetch_mbm}).

For multi-robot experiments, the baselines are the \gpu-based \curobovtwo~\cite{sundaralingam2026curobov2} and \prrtc~\cite{huang2025prrtc}, along with the \cpu-based \simd-accelerated multi-robot planner \vampmr~\cite{huang2026vampmr}. We have also implemented a \gpu-based bi-directional \aodrrt~\cite{shome2020drrt} as a multi-robot \asao baseline to compare against the \aox meta-algorithm.\footnote{Note that we chose \aodrrt, rather than \drrtstar, because the \aox vs. rewiring operations serve as the fundamental difference between \mrpop and other multi-robot \asao planners. The additional changes seen in \drrtstar, such as offline processing, can be directly added to \mrpop.}

For single-robot experiments, the baselines are the \gpu-based \curobovtwo~\cite{sundaralingam2026curobov2} and \prrtc~\cite{huang2025prrtc}, and the \cpu-based \simd-accelerated \asao planners \aorrtc~\cite{wilson2025aorrtc}, \rrtstar~\cite{karaman2010incremental}, Batch Informed Trees (\bitstar)~\cite{gammell2015batch}, and Fast Marching Tree (\fmt)~\cite{janson2015fast} from the Open Motion Planning Library 2.0~\cite{guo2026open}.

All experiments were conducted on a desktop computer with an Intel Core Ultra 9 285K 24-Core \cpu and an \nvidia GeForce RTX 5090 \gpu.

\subsection{Multi-Robot Simulation Benchmarks}
As shown in Fig.~\ref{fig:panda4_init}, only \mrpop and \prrtc achieved a 100\% solve rate for the 28 \dof four Franka bin packing problems when each planner was given 50 seconds per problem before declaring initial planning failed. Comparing against planners with a solve rate of at least 60\%: \mrpop achieved a 2.5x speedup over \prrtc, 1.9x speedup over \vampmr \cbs, and a 4.3x speedup over \vampmr \rrt. For path optimization, we benchmark against the bi-directional \gpu-based \aodrrt, which attains the \asao guarantee through the rewiring operation. From Fig.~\ref{fig:cost_vs_time_panda4}, we see that \mrpop converges to solutions of lower cost faster while having a smaller interquartile range, showcasing \mrpop's superior speed and consistent optimization performance. 

\begin{figure*}[t]
    \includegraphics[width=1\textwidth]{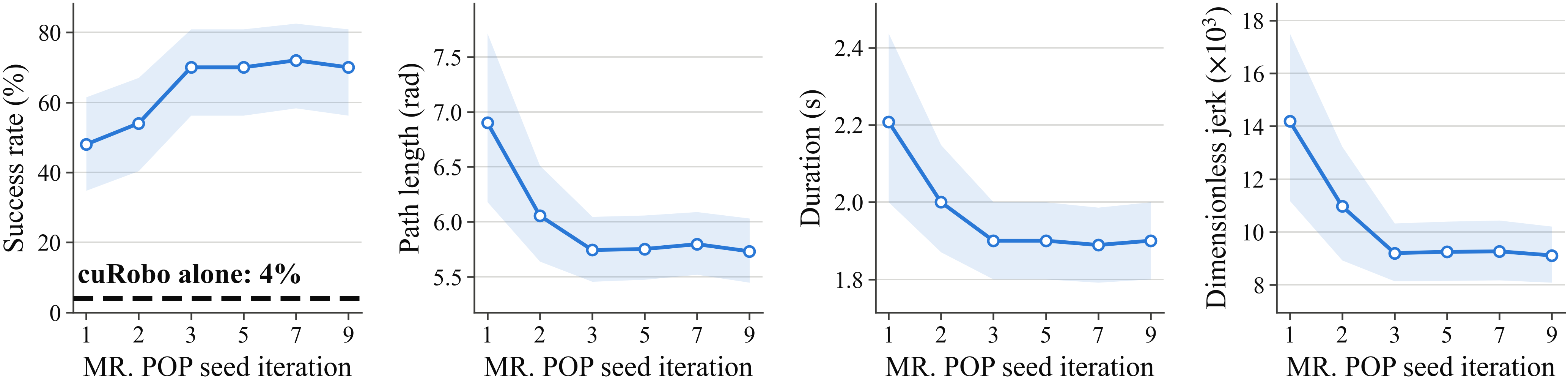}
    \vspace{-1.5em}
    \caption{
        \mrpop seeding iteration vs. \curobovtwo optimizer outcome tested on the 28 \dof four Franka arm problems. As \mrpop's iteration increased, the \curobovtwo optimizer succeeds more often while producing paths of lower cost, duration, and dimensionless jerk. 
    }
    \label{fig:seed_iter_metrics}
    \vspace{0pt}
\end{figure*}

\begin{figure*}[t]
  \centering
  \begin{subfigure}[t]{0.28\textwidth}
    \centering
    \includegraphics[width=\linewidth,keepaspectratio]{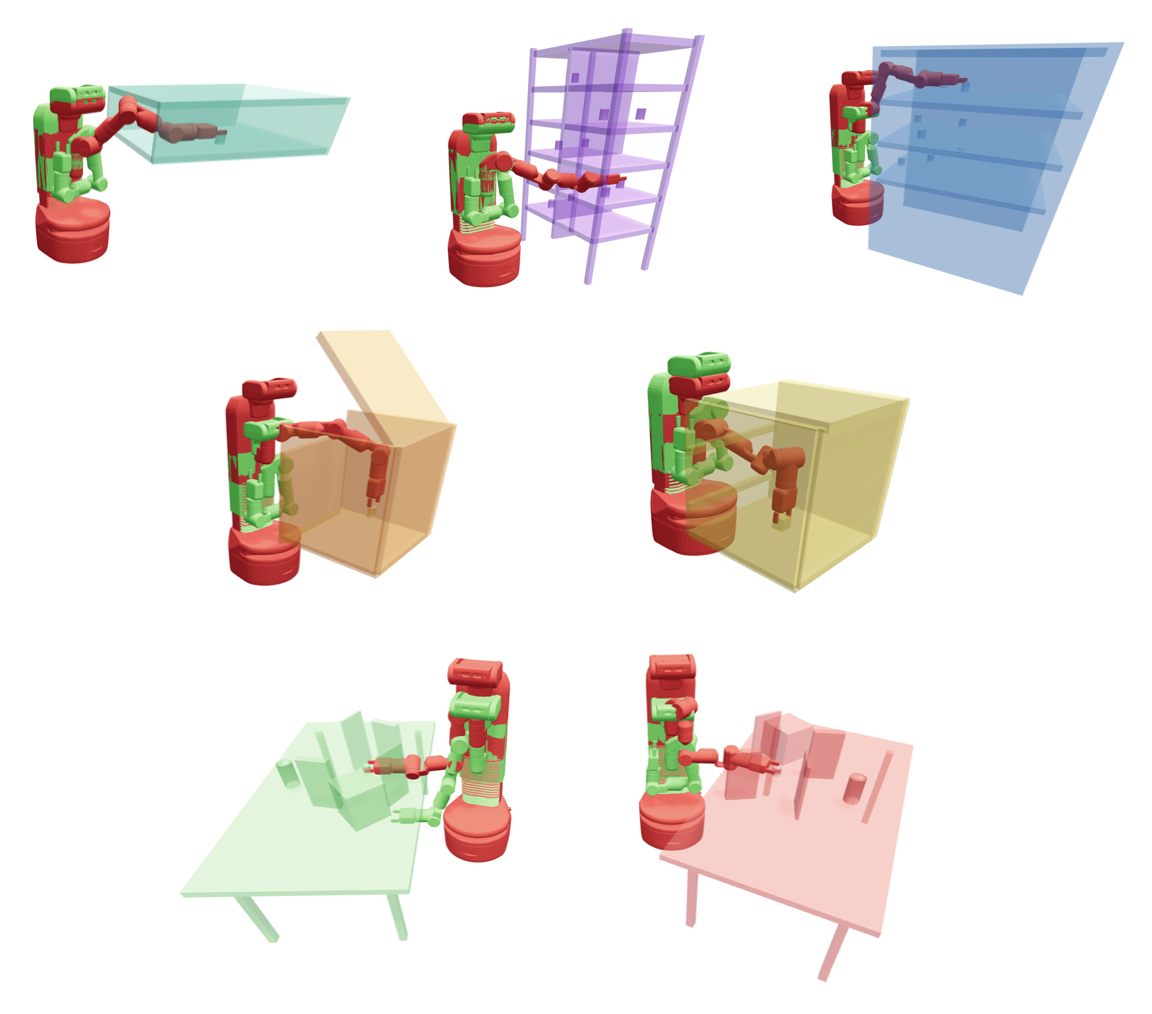}
    \caption{}
    \label{fig:fetch_mbm}
  \end{subfigure}
  \hfill
  \begin{subfigure}[t]{0.37\textwidth}
    \centering
    \includegraphics[width=\linewidth,keepaspectratio]{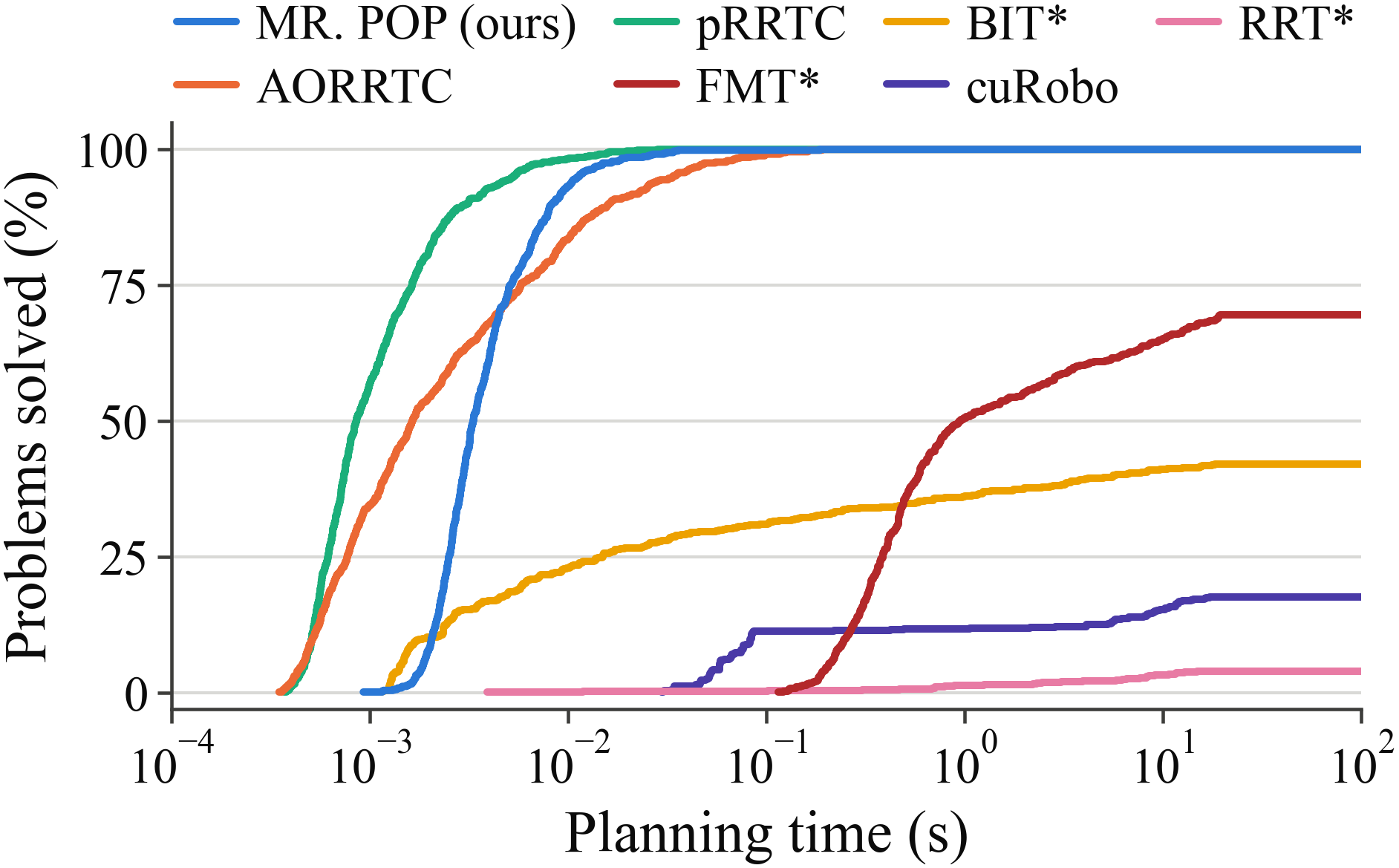}
    \caption{}
    \label{fig:fetch_init_time_success_cdf_kernel}
  \end{subfigure}
  \hfill
  \begin{subfigure}[t]{0.32\textwidth}
    \centering
    \includegraphics[width=\linewidth,keepaspectratio]{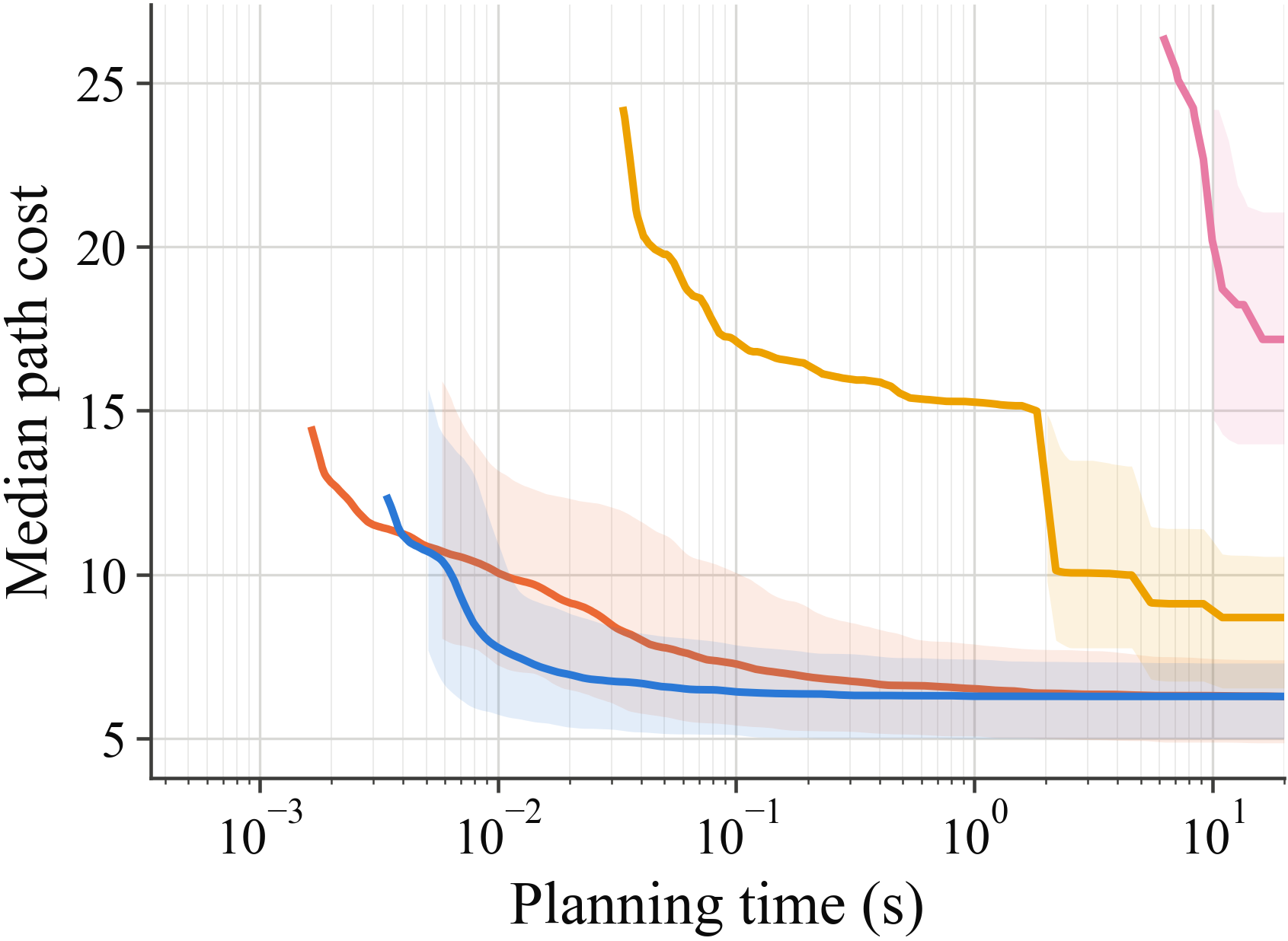}
    \caption{}
    \label{fig:fetch_time_vs_cost_own}
  \end{subfigure}
  \hfill
  \caption{Benchmark problems and results for the 8 \dof Fetch on the MotionBenchMaker dataset~\cite{chamzas2021motionbenchmaker}. \textbf{(a)} The MotionBenchMaker scenes provide diverse manipulation problems for the 8 \dof Fetch, including counter-top manipulation, accessing shelves, and constrained reaching; Fetch has 7 scenes with 100 problems each. \textcolor{green!90!white}{Green} and \textcolor{red}{red} denote the start and goal configurations per sample problem, respectively. \textbf{(b)} Among all planners, only \mrpop, \prrtc, and \aorrtc achieved a 100\% solve rate. \textbf{(c)} \mrpop converges to low-cost solutions orders of magnitude faster compared to other \asao planners, while also having a smaller interquartile range which demonstrates more consistent performance.}
  \label{fig:fetch}
\end{figure*}

For the 35 \dof five Franka benchmark (Fig.~\ref{fig:panda5_init},~\ref{fig:cost_vs_time_panda5}) we observe that \mrpop is the only planner achieving a 100\% solve rate while consistently converging to low-cost solutions faster. Compared to planners with a solve rate of at least 60\%: \mrpop averages a 3.1x speedup over \prrtc and a 5.1x speedup over \vampmr \cbs and \vampmr \rrt. Since \mrpop is the only planner that achieved a 100\% solve rate, the above means \mrpop averages a faster solve time while solving manipulation problems that other planners fail. Comparing across the 28 \dof and 35 \dof results, we see that \mrpop's solve rate lead from the second-best planner rises from 0\% to 16\%, while the average speedup also increased. This shows that \mrpop provides a more significant advantage as problem complexity scales.

Beyond geometric \asao planning, we evaluate \mrpop's effectiveness as a multi-robot seeder for motion optimizers. Seeds matter because they keep local optimizers out of poor local minima. We use \mrpop as the seeder for \curobovtwo's jerk-free motion optimizer, and compare \mrpop's optimizing iteration against \curobovtwo's output using the 28 \dof four Franka problems. The default \curobo seeder is a \prm that plans directly in the composite space, which often fails as that search space grows exponentially~\cite{shaoul2024accelerating}.

Referring to Fig.~\ref{fig:seed_iter_metrics}, we see that \mrpop raises \curobovtwo's success rate from 4\% to 72\%, with a higher optimizing iteration leading to higher success rate. With a higher \mrpop iteration count, the seeding path becomes more optimized and leads to \curobo producing paths with shorter length and motion time. The path's dimensionless jerk~\cite{hogan2009sensitivity}, which is motion jerk normalized across time and displacement length, decreases alongside increasing \mrpop iteration. 

\subsection{Single-Robot Simulation Benchmarks \& Ablations}
For single robot benchmarks and ablations, \mrpop bypasses roadmap construction and directly plans in the robot's configuration space. This isolates the performance of \mrpop's bi-directional tree search in finding initial solutions and the \aox layer for path optimization.

As shown in Fig.~\ref{fig:fetch_init_time_success_cdf_kernel}, when each planner was given 20 seconds per problem before declaring planning failed, only \mrpop, \prrtc, and \aorrtc achieved a 100\% solve rate. \mrpop lagged behind \prrtc as \mrpop in the single robot scenario functions as a GPU-based RRT-Connect, with the additional \aox overhead for \asao planning. Overall, \mrpop achieved a faster time to initial solution for harder problems but lagged \aorrtc on easier problems due to \gpu launch and \io overheads.

For the path optimization benchmark, we only retain the anytime \asao planners (\mrpop, \aorrtc, \bitstar, and \rrtstar). As a single robot uses no roadmap layer, these results isolate the \aox design. As shown in Fig.~\ref{fig:fetch_time_vs_cost_own}, \mrpop converges to paths of lower cost orders of magnitude faster than the other planners, while also having a smaller interquartile range of path cost across problems compared to \aorrtc, indicating that \mrpop is not only faster in path optimization but is also more consistent.

\subsection{Hardware Deployment}
We deploy \mrpop on a 14 \dof dual-arm system, consisting of a 7 \dof Flexiv Rizon 4s and a 7 \dof Kinova Gen 3. As shown in Fig.~\ref{fig:hardware_comparison}, \mrpop's path optimality and effectiveness as a motion optimizer seeder transfer to a real-world setup. In Fig.~\ref{fig:chrono_realtime}, we further demonstrate \mrpop's ability to react in dynamic environments and plan optimized paths around obstacles in real time. Please refer to the supplemental video for more hardware deployment details and visualizations.
\iffigs
\begin{figure}[!t]
   \centering
   \vspace{1em}
   \includegraphics[width=0.48\textwidth]{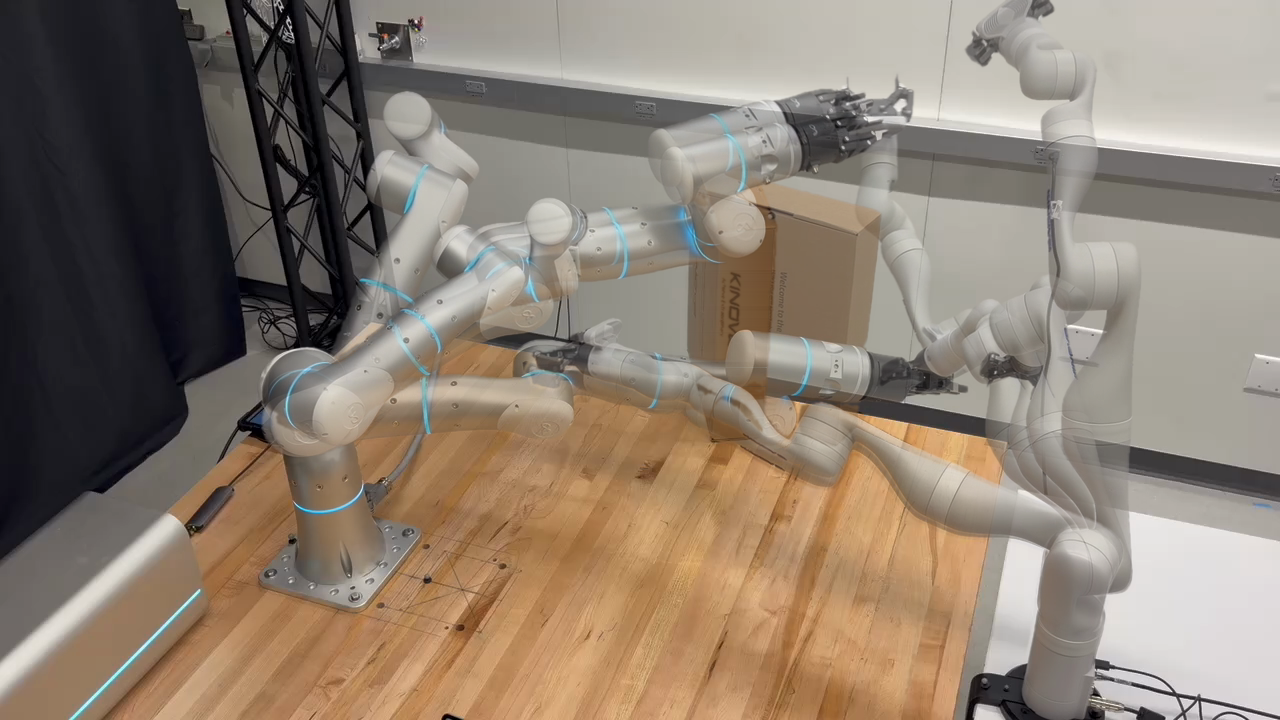}
   \vspace{-5pt}
   \caption{\mrpop replanning for a 14 \dof dual arm amid moving obstacles.}
   \label{fig:chrono_realtime}
   \vspace{-15pt}
\end{figure}
\fi

\section{Conclusion and Future Work} \label{sec:conclusion}
Finding globally optimal paths remains a fundamental challenge in multi-robot motion planning. Despite acceleration of almost-surely asymptotically optimal (\asao) planners via \cpu-based parallelism, achieving both probabilistic convergence guarantees and strong computational performance, these algorithms still struggle to scale to multi-robot settings. 
As such, we introduce \mrpop, a \gpu-based \asao multi-robot planner based on \drrt and the \aox meta-algorithm. 
\mrpop uses large-scale \gpu-based \simt-parallelism to simultaneously run hundreds of roadmap construction and tree search iterations with underlying parallel nearest neighbor search and collision checking operations.
We show that this enables \mrpop to become the only planner achieving a 100\% solve rate while being faster than state-of-the-art \asao planners in multi-robot systems up to 35-\dof. \mrpop also raises the success rate of downstream motion optimizers (e.g., from 4\% to 72\%), by creating high-quality, diverse seeds that help avoid local minima.

\section{Acknowledgment} \label{sec:acknowledgment}
We used generative AI tools, such as Codex and Claude Code, to help with software implementation, data visualization, and manuscript formatting. All final contents are modified and reviewed by humans.

\printbibliography

\end{document}